\documentclass[letterpaper, 10 pt, conference]{ieeeconf}

\IEEEoverridecommandlockouts
\usepackage{algorithm}
\usepackage{algpseudocode}
\usepackage{cite}
\usepackage{graphics} % for pdf, bitmapped graphics files
\usepackage{epsfig} % for postscript graphics files
\usepackage{mathptmx} % assumes new font selection scheme installed
\usepackage{times} % assumes new font selection scheme installed
\usepackage{amssymb}  % assumes amsmath package installed
\usepackage[cmex10]{amsmath}
\usepackage{amsthm}
\usepackage{amsfonts}
\usepackage{latexsym}
\usepackage{mathrsfs}
\usepackage{xcolor}
\usepackage{svg}
\svgpath{{../figures/}}
\usepackage[normalem]{ulem}
\usepackage{mathtools}
\usepackage{xspace}
\usepackage{geometry}
\usepackage{graphicx}
\usepackage{subcaption}

\algblock{Input}{EndInput}
\algnotext{EndInput}
\algblock{Output}{EndOutput}
\algnotext{EndOutput}
\newcommand{\Desc}[2]{\State \makebox[2em][l]{#1}#2}
\algnewcommand\algorithmicforeach{\textbf{for each}}
\algdef{S}[FOR]{ForEach}[1]{\algorithmicforeach\ #1\ \algorithmicdo}

\usepackage[linkbordercolor={1 1 1},citebordercolor={1 1
  1},urlbordercolor={0.0 0.0
  0.0},urlcolor=blue,colorlinks=true,linkcolor=black,citecolor=black]{hyperref}
\usepackage{url}

\theoremstyle{remark}

\theoremstyle{definition}
\newtheorem{defn}{Definition}

\newtheorem{ass}{Assumption}

\theoremstyle{plain}

\newtheorem{thm}{Theorem}
\newtheorem{prop}{Proposition}  
\newtheorem{example}{Example}

\newcommand{\R}{\mathbb{R}}
\renewcommand{\S}{\mathit{S}}
\newcommand{\norm}[1]{\left\lVert#1\right\rVert}

\title{\LARGE \bf
Safety-Critical Manipulation for Collision-Free Food Preparation 
}

\author{Andrew Singletary, William Guffey, Tamas G. Molnar, Ryan Sinnet, and Aaron D. Ames% <-this % stops a space
\thanks{
Andrew Singletary, Tamas G. Molnar and Aaron D. Ames are with Department of Mechanical and Civil Engineering, California Institute of Technology, Pasadena CA 91125, USA. Email addresses:
		{\tt \small \{asinglet, tmolnar, ames\}@caltech.edu}. 
}
\thanks{
William Guffey and Ryan Sinnet are with Miso Robotics, Pasadena, CA 91101, USA. Email addresses:
		{\tt \small \{wguffey, rsinnet\}@misorobotics.com}.
}
\thanks{
This work is supported by Miso Robotics and NSF CPS award \#1932091.
}
}

\begin{document}

\maketitle
\thispagestyle{empty}
\pagestyle{empty}

\begin{abstract}
Recent advances allow for the automation of food preparation in high-throughput environments, yet the successful deployment of these robots requires the planning and execution of quick, robust, and ultimately collision-free behaviors. In this work, we showcase a novel framework for modifying previously generated trajectories of robotic manipulators in highly detailed and dynamic collision environments using Control Barrier Functions (CBFs). This method dynamically re-plans previously validated behaviors in the presence of changing environments---and does so in a computationally efficient manner. Moreover, the approach provides rigorous safety guarantees of the resulting trajectories, factoring in the true underlying dynamics of the manipulator. This methodology is extensively validated on a full-scale robotic manipulator in a real-world cooking environment, and has resulted in substantial improvements in computation time and robustness over re-planning.
\end{abstract}

\section{Introduction}
\label{sec:introduction}
Robotics and automation
% have completely transformed manufacturing over the past several decades, and they are beginning to transform the food industry.
have great potential to transform the food industry.
In the domain of autonomous cooking, robotic manipulators are used to pick up, deep fry, and dispense the food in the dynamic environment of the kitchen.
% In the domain of autonomous cooking, new motion plans are constantly computed for robotic manipulators that pick up, deep fry, and dispense the food in the dynamic environment that is the kitchen.
This requires motion plans that are constantly computed, hundreds or thousands of times per day, subject to different environmental factors and initial conditions of the robots. Due to the extremely complex collision environments and non-trivial kinematics, highly non-linear planning algorithms such as TrajOpt \cite{schulman2014motion}, OMPL \cite{sucan2012open}, and CHOMP \cite{ratliff2009chomp} are used to plan joint trajectories offline, which the manipulator then executes. The vast majority of plans, however, deviates only slightly from previously computed trajectories: food baskets may shift locations and deform slightly, workers may push the equipment, or the robot may have slightly different joint configuration initially. In these situations, rather than re-planning a trajectory with the existing motion planner, we propose a safety filtering method that produces collision-free trajectories from existing reference trajectories in minimal computation time, and with formal safety guarantees.

Minimally modifying existing trajectories 
% subject to small changes in state and environment should not pose a large challenge.
is possible by optimization solvers that have warm-start or hot-start options for resolving problems with slightly modified initial conditions.
In \cite{lembono2020memory}, the authors introduced a method for building a dataset of motion plans that were used to warm-start the trajectory generator to boost the success-rate of trajectories. Similarly, in \cite{banerjee2020learning}, the authors proposed a dataset of expert trajectories to warm-start a Sequential Convex Programming (SCP) problem for solving locally optimal trajectories rapidly. In \cite{dong2016motion}, the authors used incremental solvers to update trajectories via Gaussian processes and factor graphs. 

More generally, local planners have been used for decades to modify rough, global trajectories under new collision constraints \cite{baginski1996local} or dynamic environments \cite{terasawa2016achievement}. While many of these works could certainly be modified to tackle the robotic cooking problem, we believe that our approach's balance of simplicity, computational speed, and formality of resulting safety guarantees makes it the best fit for the problem at hand. Moreover, this algorithm can be run in real-time as a feedback controller with dynamically updating environments, offering a great deal of flexibility in implementation.

Our approach relies on control barrier functions (CBFs) \cite{ames2017cbf}, that have been proven to provide an effective means of enforcing safety on a wide variety of robotic systems \cite{ames2019control}, including robotic manipulators \cite{landi2019safety,singletary2019online,singletary2021safety}.
In prior works, CBFs were utilized as safety filters on desired velocity commands, and obstacle representations were simplified.
In this work, safe velocity commands synthesized based on kinematics are tracked by low-level controllers, and a formal proof is provided that this method preserves safety for the full dynamics of the robot.
The formal connection between safe kinematics and dynamics leads to a theoretically justified framework for filtering pre-computed trajectories.
Furthermore, our work utilizes significantly more complex obstacle representations and environments than previous works involving CBFs, which facilitates practical implementation.
% In this work, we provide a framework for filtering pre-computed trajectories that utilizes significantly more complex obstacle representations and environments than previous works involving CBFs.
% Furthermore, we prove that the safety guarantees achieved by designing velocity commands based on kinematics hold also for the full dynamics of the robot that tracks these velocities.

\begin{figure}[t]
\includegraphics[width=\columnwidth]{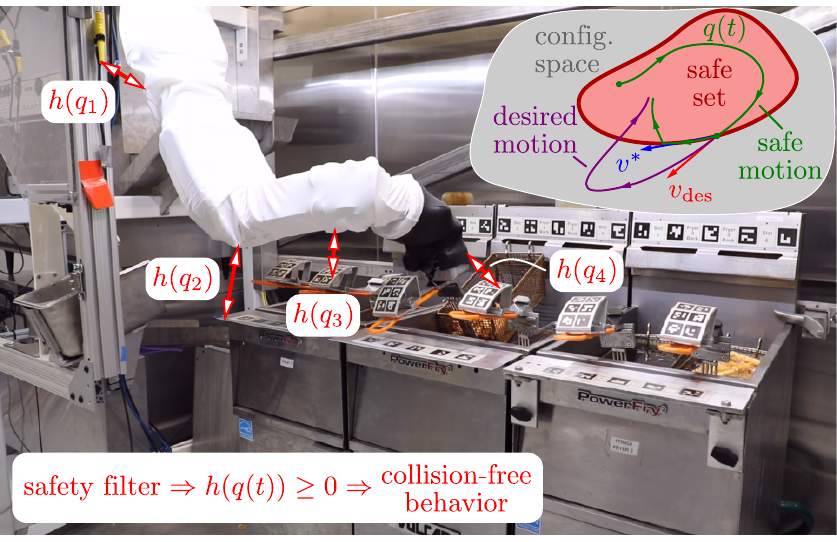}
\caption{Miso Robotics "Flippy2" robot frying food using our proposed safety-critical framework for food preparation.}
\vspace{-.08cm}
\end{figure}

% The primary contribution of this work is a rigorously tested CBF-based filtering strategy that modifies previously generated trajectories subject to new collision constraints and initial conditions. 

The primary contribution of this work is a rigorously tested CBF-based filtering strategy that modifies previously generated trajectories to account for new collision constraints in a provably safe manner. 
This strategy often eliminates the need for re-planning in updated environments, saving computation time and providing robust safety guarantees for the resulting trajectory. We formally prove that these trajectories are not only valid for the kinematic model of the manipulator, but also for the underlying full-order dynamical system.
%Moreover, an algorithm for intelligently choosing the optimal initial trajectory in a previously-generated cache is presented, including a criteria for adding these modified trajectories to the cache.
The proposed novel control algorithm is implemented in the MoveIt framework \cite{coleman2014reducing}, and applied to full-scale autonomous food-frying in collaboration with Miso Robotics.
The speed and efficacy of this method are extensively explored in real-world cooking environments, and the method has been shown to dramatically increase planning speed and reliability.

The layout of this paper is as follows.
In Section \ref{sec:motivation}, CBFs are used to enforce safety on both the kinematic model of the manipulator and the full dynamics.
Section \ref{sec:theory} formulates distance functions in complex, real-world environments, which are used in the context of CBFs for collision avoidance.
Section \ref{sec:implementation} outlines the software implementation of the proposed algorithm and the simulation environment.
Lastly, Section \ref{sec:results} shows the details and results of the extensive, real-world hardware tests in the application of robotic cooking.

\section{Control Barrier Functions for Safety}
\label{sec:motivation}
\subsection{Background: Control Barrier Functions}

Consider a nonlinear system in control-affine form:
\begin{align} \label{eqn:dyn}
    \dot{x} = f(x) + g(x) u,
\end{align}
with state $x \in \R^k$ and control input $u \in U \subset \R^m$ to be chosen from an admissible input set $\mathit{U} \subseteq 
\R^m$.
The functions $f: \R^k \to \R^k$ and $g: \R^k \to \R^{k \times m}$ describe the dynamics of the system and are assumed to be Lipschitz continuous.
Given a Lipschitz continuous control law $k:\R^k\to \R^m$, $u = k(x)$ we obtain the closed-loop dynamics
\begin{align} \label{eqn:cl_dyn}
    \dot{x} = f_{\rm cl}(x) := f(x) + g(x)k(x).
\end{align}
For the initial condition $x(t_0)=x_0 \in \R^k$, this system has a unique solution $x(t)$ which we assume to exist for all $t \geq t_0$.

Consider a safe subset of the state-space $\S \subset \R^k$ which may represent, for example, the collision-free states of a manipulator.
To guarantee safety, we must ensure that the state of the closed-loop system is kept within in $\S$ for all time. This is formalized through the notion of set invariance. 

\begin{defn}
The set $\mathit{S}$ is \textit{forward invariant} if the solution $x(t)$ of system \eqref{eqn:cl_dyn} satisfies $x(t) \in \mathit{S}$, $\forall t \geq t_0$. 
% Moreover, the set $\mathit{S}$ is \textit{control invariant} the control law $k(x)$ maps all $x\in \S$ to the set of admissible inputs $\textit{U}$, such that under $k$, the system is invariant, i.e. $\forall t \geq 0, \forall x_0\in\mathit{S},\ \Phi_{k}(x_0,t)\in\mathit{S}$. Here, $\Phi_{k}(x_0,t)$ denotes the flow of the system under control law $k$ from initial condition $x_0$ at time $t$.
\end{defn}

Control barrier functions are a common tool to synthesize controllers that enforce forward invariance for a given set $\S$.

\begin{defn}[\hspace{-.1pt}\cite{ames2017cbf}]
\label{def:cbf}
Let $\S \subset \R^k$ be defined as the 0-superlevel set of a continuously differentiable function $h: \R^k \to \R$:
\begin{equation}
\S = \{ x \in \R^k ~ : ~ h(x) \geq 0 \}.
\label{eqn:safeset}
\end{equation}
% \begin{eqnarray}
% \S & = & \{ x \in \R^k ~ : ~ h(x) \geq 0 \} ,
% \nonumber\\
% \partial \S & = & \{ x \in \R^k ~ : ~ h(x) = 0 \}, \nonumber\\
% \mathrm{Int}(\S) & = & \{ x \in \R^k ~ : ~ h(x) > 0 \}. \nonumber
% \label{eqn:safeset}
% \end{eqnarray}
Function $h$ is a \textit{control barrier function (CBF)} for \eqref{eqn:dyn} on $\S$ if there exists an \emph{extended class $\mathit{K}_{\infty}$ function}\footnote{$\alpha : \R \to \R$ is an extended class $\mathit{K}_{\infty}$ function if it is continuous, strictly monotonically increasing, and satisfies $\alpha(0) = 0$, $\lim_{r\to\infty}\alpha(r) = \infty$ and $\lim_{r\to-\infty}\alpha(r) = -\infty$.}
$\alpha$ such that for all 
$x \in \S$: %$\exists u$ s.t. 
\begin{align}
\label{eqn:cbf:definition}
% \exists ~  u \in U \quad \textrm{s.t.} \quad 
\sup_{u \in U} 
\underbrace{
\left[
\frac{\partial h}{\partial x} f(x) + \frac{\partial h}{\partial x} g(x) u
\right]
}_{\dot{h}(x,u)}
\geq - \alpha(h(x)),
\end{align}
where $\dot{h}(x,u)$ is the derivative of $h(x)$ along system \eqref{eqn:dyn}.
% that is a linear function of $u$.
\end{defn}

This definition yields the following key result for CBFs.

\begin{thm}[\hspace{-.1pt}\cite{ames2017cbf}]
\label{thm:CBF}
If $h$ is a CBF for \eqref{eqn:dyn}, then any locally Lipschitz continuous controller $k:\R^k\to \R^m$, $u = k(x)$ satisfying 
$$
% u = k(x) \in \{ u \in U ~ : ~ \dot{h}(x,u) \geq - \alpha(h(x))\},
\dot{h}(x,k(x)) \geq - \alpha(h(x)) 
$$
renders the set $\S$ in \eqref{eqn:safeset} forward invariant for the resulting closed loop system \eqref{eqn:cl_dyn}.  
\end{thm}

This condition can be incorporated into a quadratic program (QP) to synthesize pointwise optimal and safe controllers, by minimally modifying a desired but not necessarily safe input $u_{\rm des}(x,t) \in U$ to a safe input $u^*(x,t) \in U$:
\begin{align}
\label{eqn:QPsimple}
\begin{split}
u^*(x,t) = \underset{u \in \mathit{U}}{\operatorname{argmin}} & ~  ~ {\| u - u_{\rm des}(x,t) \|}_2^2 \\
% \mathrm{s.t.} & ~  ~ \frac{\partial h}{\partial x} f(x) + \frac{\partial h}{\partial x} g(x) u \geq - \alpha (h(x)).
\mathrm{s.t.} & ~  ~ \dot{h}(x,u) \geq - \alpha (h(x)).
\end{split}
\end{align}
This QP can be solved in real-time for nonlinear systems.

\subsection{Application to Robotic Manipulators}

Now let us use CBFs for controlling robotic manipulators whose state $x=(q,\dot{q})$ consists of the configuration $q \in \R^n$ and the joint velocities $\dot{q} \in \R^n$.
% For the scope of this paper, we only consider the kinematics of the manipulators. That is, we consider the dynamic system:
% To begin our analysis, we consider the kinematics of manipulators:
For safe obstacle avoidance with the manipulator, we consider the safe set to be defined over the configuration space:
\begin{equation}
\S = \{ q \in \R^n ~ : ~ h(q) \geq 0 \},
\label{eqn:safeset_robot}
\end{equation}
where $h : \R^n \to \R$ is continuously differentiable.
That is, $h$ is assumed to be independent of $\dot{q}$.
The specific choice of $h$ will be given in Section~\ref{sec:synthesis}.

First, we consider the kinematics of robotic manipulators with state $q$ --- later it will be formally justified how this yields safety guarantees on the full-order dynamics with state $(q,\dot{q})$. 
In particular, we consider the system:
\begin{equation}
\label{eqn:kinematic_model}
    \dot{q} = v,
\end{equation}
wherein we assume direct control over the joint velocities via the commanded velocity $v \in \R^n$. 
% For high-speed applications, it may be important to consider the dynamics, as done in \cite{singletary2021safety}, but they are neglected in this paper.
We design a velocity $v$ by considering it as input to system \eqref{eqn:kinematic_model} and guaranteeing safety by CBFs.
In Section~\ref{sec:kin2dyn}, it will be verified that safety guarantees extend to the full dynamics when the commanded velocity is tracked by a low-level controller.
% (i.e., when $\dot{q} \to v$ as $t \to \infty$ rather than $\dot{q}=v$). 

Because each joint's velocity is directly controlled according to \eqref{eqn:kinematic_model}, we can simplify the QP shown in \eqref{eqn:QPsimple} to:
\begin{align}
\label{eqn:QPkinematic}
\begin{split}
v^*(q,t) = \underset{v \in \R^n}{\operatorname{argmin}} & ~  ~ {\| v - v_{\rm des}(q,t) \|}_2^2 \\
\mathrm{s.t.} & ~  ~ \frac{\partial h}{\partial q} v \geq - \alpha h(q),
\end{split}
\end{align}
where a desired velocity $v_{\rm des}(q,t) \in \R^n$ is modified to a safe velocity $v^*(q,t) \in \R^n$.
Note that we chose the extended class~$\mathit{K}_{\infty}$ function to be linear with constant gradient $\alpha > 0$.

\begin{example} \label{ex:simple}
Consider a 6-degrees-of-freedom manipulator ($n=6$) with a spherical tool attachment of radius $r_1$. The manipulator is intended to track a desired joint velocity $v_{\rm des}(q,t)$ and we wish to avoid a spherical region centered at $O \in \R^3$ of radius $r_2$.
The CBF can be written as the distance from the spherical tool to the sphere in the surroundings:
\begin{align}
    h(q) &= \norm{F(q) - O}_2 - (r_1+r_2)\\
    &=  \sqrt{(F_x -O_x)^2+(F_y -O_y)^2+(F_z -O_z)^2} - (r_1+r_2) \nonumber
\end{align}
where $F : \R^6 \rightarrow \R^3$ are the forward kinematics that give the position of the end-effector in space, $(F_x,F_y,F_z) = F(q)$.
The gradient of the CBF can be computed as:
\begin{equation}
    \frac{\partial h}{\partial q} = \frac{\partial h}{\partial F} \frac{\partial F}{\partial q} =
    \frac{1}{\norm{F(q) - O}_2}
    \begin{bmatrix}
    F_x-O_x \\ 
    F_y-O_y \\
    F_z-O_z
    \end{bmatrix}^T J(q),
\end{equation}
where $J : \R^6 \to \R^3 \times \R^6$, $J(q) = \frac{\partial F}{\partial q}$ is the top three rows of the manipulator Jacobian.
By enforcing the CBF-QP \eqref{eqn:QPkinematic}, we obtain the path illustrated in Figure~\ref{fig:simple_example}.
\end{example}

\begin{figure}[t]
\includegraphics[trim={0 2cm 0 3cm},clip,width=\columnwidth]{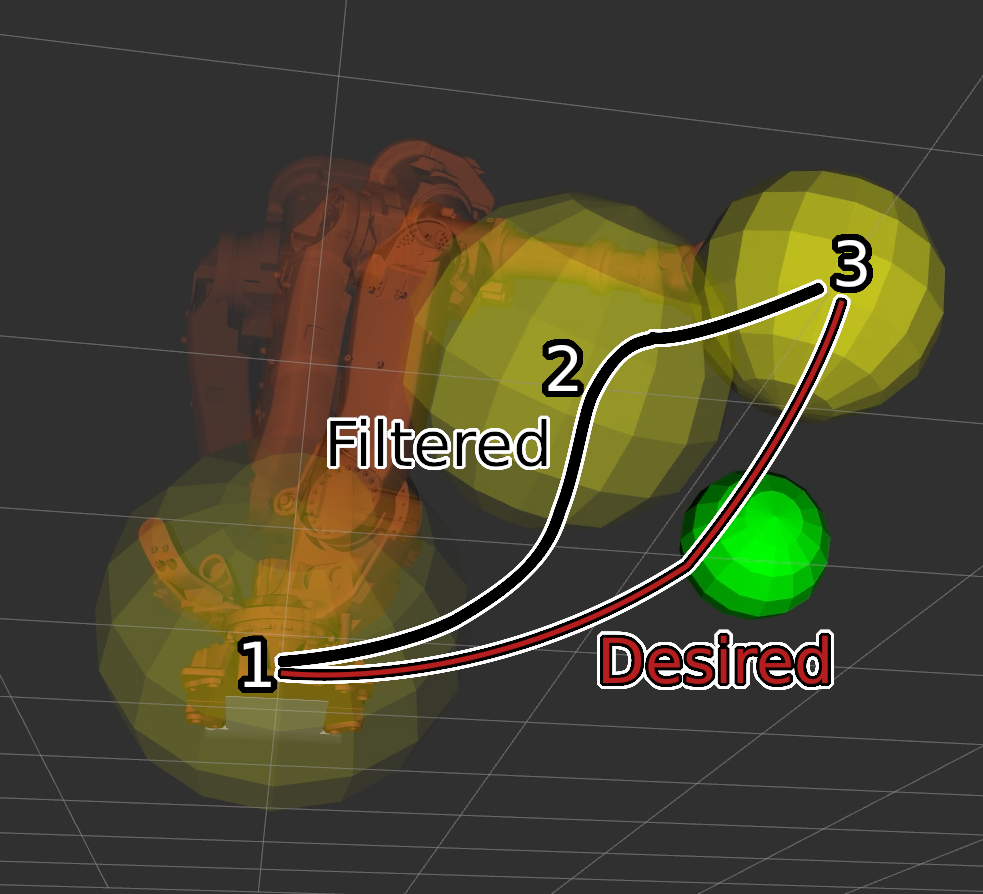} 
\caption{Manipulator trajectory resulting from the control barrier function detailed in Example \ref{ex:simple}. The tool is marked in yellow, whereas the obstacle is shown in green.}\label{fig:simple_example}
\end{figure}

\subsection{Safety Guarantees: from Kinematics to Dynamics}
\label{sec:kin2dyn}

We now establish the first theoretic contribution of the paper that will serve to formally justify the subsequent results.
In particular, we leverage the kinematic model of the manipulator to guarantee safe behavior on the full-order dynamics.
We establish that tracking the safe velocity obtained from the QP \eqref{eqn:QPkinematic} results in safety under reasonable conditions on the tracking controller. 

% The safety guarantees given by Proposition \ref{prop:issf} were applicable to the kinematic model of a robot manipulator: $\dot{q} = u$. 
% The goal of this section is to establish that these same control methodology gives safety guarantees on the full-order dynamics assuming good tracking of the desired velocity. 

Specifically, consider the full-order dynamics associated with a robotic manipulator \cite{murray2017mathematical}: \begin{equation}
\label{eqn:robot}
        D(q) \dot{q} + C(q,\dot{q})\dot{q}+G(q) = Bu,
\end{equation}
with $q,\dot{q} \in \R^n$, $D(q) \in \R^{n \times n}$ the inertia matrix, $C(q,\dot{q}) \in \R^{n \times n}$ the Coriolis matrix, and $G(q) \in \R^n$ the gravity vector. Here we assume full actuation: the actuation matrix $B \in \R^{n \times n}$ is invertible and $u \in \R^n$.  Associated with these dynamics is a control system of the form \eqref{eqn:dyn} with $x = (q,\dot{q})$ (hence $k = 2n$). 

Motivated by the approach in \cite{molnar2021model}, we assume the existence of a ``good'' low-level velocity tracking controller on the manipulator (as is common on industrial robots).  Concretely, for a velocity command $v^*(q,t)$ consider the corresponding error in tracking this velocity: 
\begin{equation}
\dot{e} = \dot{q} - v^*,
\label{eqn:error}
\end{equation}
and assume exponentially stable tracking.

\begin{ass} \label{ass:tracking}
There exist a low-level controller $u = k(x,t)$ for the control system \eqref{eqn:dyn} obtained from \eqref{eqn:robot} such that
\begin{eqnarray}
\label{eqn:expconvergence}
{\| \dot{e}(t)\|}_2 \leq M e^{- \lambda t} {\| \dot{e}_0 \|}_2
\end{eqnarray}
holds for some $M,\lambda > 0$ along the solution $x(t)$ of the closed-loop system \eqref{eqn:cl_dyn} with $q(t_0)=q_0$, $\dot{q}(t_0)=\dot{q}_0$ and $\dot{e}(t_0)=\dot{e}_0$.
\end{ass}

Under this assumption, we have the first theoretic result of the paper which we state in general terms before applying it to the case of avoiding collisions in Section \ref{sec:fullordercollision}.

\begin{thm}
\label{thm:kintofull}
Consider the full-order dynamics of a robot manipulator \eqref{eqn:robot} expressed as the control system \eqref{eqn:dyn}, and the safe set $\S$ in \eqref{eqn:safeset_robot}.
Assume that $h$ has bounded gradient, i.e., there exists $C_h>0$ s.t. ${\left\| \frac{\partial h}{\partial q}\right\|}_2 \leq C_h$ for all $q \in \S$.
% $\S = \{ x = (q,\dot{q}) \in R^{2n} ~ : ~ h(q)  \geq 0 \}$
% associated with the control barrier function $h : \R^n \to \R$ dependent only on $q$.
Let $v^*(q,t)$ be the safe velocity given by the QP \eqref{eqn:QPkinematic}, with corresponding error in \eqref{eqn:error}.
% $\dot{e} = \dot{q} - v^*$.
% Assume that there exists a tracking controller $u = k(q,\dot{q})$ that tracks the desired velocity exponentially: 
% \begin{eqnarray}
% \label{eqn:expconvergence}
% \| \dot{e}(t)\| \leq M e^{- \lambda t} \| \dot{e}(0) \|
% \end{eqnarray}
% or $M, \lambda > 0$.
If Assumption~\ref{ass:tracking} holds with
$\lambda > \alpha$, safety is achieved for the full-order dynamics \eqref{eqn:robot} in that: 
\begin{eqnarray}
\label{eqn:SMsafety}
(q_0,\dot{e}_0) \in \S_M ~ \Rightarrow ~ q(t) \in S, \quad \forall t \geq t_0,
\end{eqnarray}
where:
\begin{equation}
% h_M (q,\dot{e}) & := &  h(q) - \frac{C_h}{\lambda - \alpha} M \| \dot{e} \| 
%  \\
\S_M  = \left\{ (q,\dot{e}) \in \R^{2n} ~ : ~ h(q) - \frac{C_h M}{\lambda - \alpha} \| \dot{e} \|_2 \geq 0 \right\}.
\end{equation}
\end{thm} 

\begin{proof}
First, we lower-bound $\dot{h}(q,\dot{q})$ as follows:
\begin{align}
\begin{split}
\dot{h}(q,\dot{q})
& = \frac{\partial h}{\partial q} v^* + \frac{\partial h}{\partial q} \dot{e} \\
& \geq -\alpha h(q) - {\left\| \frac{\partial h}{\partial q} \right\|}_2 {\| \dot{e} \|}_2 \\
& \geq -\alpha h(q) - C_{h} M {\|\dot{e}_0\|}_2 {\rm e}^{-\lambda t},
\end{split}
\end{align}
where we used (i) the definition \eqref{eqn:error} of the tracking error; (ii) the constraint on the safe velocity in \eqref{eqn:QPkinematic} and the Cauchy-Schwartz inequality; and (iii) the upper bound $C_{h}$ on ${\| \frac{\partial h}{\partial q} \|}_2$ and the upper bound \eqref{eqn:expconvergence} on the tracking error.
Then, consider the following continuous function ${y: \mathbb{R} \to \mathbb{R}}$:
\begin{equation}
y(t) = \left( h(q_0) - \frac{C_{h} M {\|\dot{e}_0\|}_2}{\lambda - \alpha} \right) {\rm e}^{-\alpha t} + \frac{C_{h} M {\|\dot{e}_0\|}_2}{\lambda - \alpha} {\rm e}^{-\lambda t},
\end{equation}
which satisfies:
\begin{align}
\begin{split}
\dot{y}(t) & = - \alpha y(t) - C_{h} M {\|\dot{e}_0\|}_2 {\rm e}^{-\lambda t} \\
y(t_0) & = h(q_0).
\end{split}
\end{align}
For ${(q_0,\dot{e}_0) \in S_{M}}$, we have ${y(t) \geq 0}$, ${\forall t \geq t_0}$, and by the comparison lemma we get:
\begin{equation}
h(q(t)) \geq y(t) \geq 0, \quad \forall t \geq t_0,
\end{equation}
that implies ${q(t) \in S}$, ${\forall t \geq t_0}$.
This completes the proof.
\end{proof}

% % \AAcomment{Here is where I would add a discussion on how velocity tracking yields safety guarantees for the full order dynamics, by giving a summary of the main theorem of \cite{molnar2021model}.  We can state that result here, then return to it later in the paper in the context of the main result of this work (proposition 1). }
% In \cite{molnar2021model}, the authors propose a technique for providing safety guarantees for the full-order dynamics of robotic systems that are tracking velocity outputs from a CBF. 

% Suppose that that a low-level controller $u = k(q,\dot{q})$ is able to track a desired velocity $v^*$ with exponential stability: 
% \begin{equation}\label{eqn:exp_tracking}
% \norm{\dot{q(t)}-v^*(t)} \leq M \norm{\dot{q(t_0)}-v^*(t_0)}\exp{(-\lambda t)}
% \end{equation}
% for some $M,\lambda > 0$.
% The main result of their work is summarized in the following theorem:

% \begin{thm}[\hspace{-.1pt}\cite{molnar2021model}]
% \label{thm:model_free}
% For the robotic system with dynamics
% \begin{equation}
%         D(q) \dot{q} + C(q,\dot{q})\dot{q}+G(q) = Bu,
% \end{equation}
% if the low-level controller satisfies \eqref{eqn:exp_tracking} for the velocity output of \eqref{eqn:QPkinematic},
% \end{thm}

\section{Distance Functions and Safety Filtering}
\label{sec:theory}
\subsection{Collisions with Environment}
In order to prevent collisions with the environment, we must ensure that any point on the robot does not come into contact with any point in the environment. However, unlike the simple example before, we cannot rely on the robot and environment being represented by simple spheres.

Let us denote the set of all points on the robot as $A \subset \R ^3$, and the set of all points in the collision environment as $B \subset \R ^3$. To guarantee safety,
% with respect to the environment,
we require that $A \cap B = \emptyset$, thus $\textrm{distance}(A,B) > 0$. More formally, \emph{distance} is defined as:
\begin{equation}
    % \textrm{distance}(A,B) = \inf \left\{\norm{p_A-p_B} \mid p_A \in A, p_B \in B\right\},
    \textrm{distance}(A,B) = \inf_{\substack{p_A \in A \\p_B \in B}} \norm{p_A-p_B}_2,
\end{equation}
which can be computed in $\R^3$ using the GJK algorithm \cite{gilbert1988fast}. 

This notion gives a nonnegative distance, which could be used as CBF.
However, it is advantageous to define a CBF that is negative in the event of collision, since CBFs may also ensure that the boundary of the set $\S$ is re-approached if $h(x) < 0$ \cite{ames2017cbf}.  
In collision, \emph{penetration} is defined as:
\begin{equation}
    % \textrm{penetration}(A,B) = \inf \left\{\norm{p_A-p_B} \mid p_A \in A, p_B \in \overline{B}\right\},
    \textrm{penetration}(A,B) = \inf_{\substack{p_A \in A \\p_B \in \overline{B}}} \norm{p_A-p_B}_2,
\end{equation}
where $\overline{B}$ is the complement of $B$, or the set of points outside the collision scene. Penetration is often computed using the EPA algorithm \cite{van2001proximity}.

These two functions can be combined to form the notion of \emph{signed distance}. 
Signed distance is typically written as 
\begin{equation}
    \mathrm{sd}(A,B) = \textrm{distance}(A,B) - \textrm{penetration}(A,B).
\end{equation}
% This concept is demonstrated visually in Figure \ref{fig:sd}.
% \begin{figure}[t]
% \includegraphics[width=\columnwidth]{example-image-c}
% \caption{Illustration of signed distance.} \label{fig:sd}
% \end{figure}
% Denote the closest points from the robot and collision object as $p_A$ and $p_B$, given in local coordinates.
When the points $p_A$ and $p_B$ of the robot and the environment are given in local coordinates, the following expression from \cite{schulman2014motion} can be utilized to compute the signed distance:
\begin{equation}
    \mathrm{sd}_{AB}(q) = \max_{\substack{\tilde{n} \in \R^3 \\ \norm{\tilde{n}}_2 = 1}}\min_{\substack{p_A \in A \\p_B \in B}}\tilde{n}\cdot \left(F_A^{\rm W}(q) p_A - F_B^{\rm W} p_B\right),
    \label{eqn:sd_maxmin}
\end{equation}
where $F_A^{\rm W}(q) \in \R^{3 \times 3}$ gives the pose of the robot in the world frame that depends on the configuration $q$, and $F_B^{\rm W} \in \R^{3 \times 3}$ gives the pose of the collision environment, i.e., $F_A^{\rm W}(q) p_A$ and $F_B^{\rm W} p_B$ indicate points in the world frame.

\subsection{Controller Synthesis with Control Barrier Functions}
\label{sec:synthesis}

Given the signed distance, we propose the CBF candidate: 
\begin{equation} \label{eq:cbf_sd}
    h(q) = \mathrm{sd}_{AB}(q),
\end{equation}
which defines the corresponding safe set of the system: 
\begin{equation} \label{eq:cbf_sd_set}
   \S = \{q \in \R^n ~ : ~  h(q) = \mathrm{sd}_{AB}(q) \geq 0\}. 
\end{equation}
We remark that based on \eqref{eqn:sd_maxmin} $h$ can be written as:
\begin{equation}
    h(q) = \hat{n}(q)^\top \left(F_A^{\rm W}(q) \hat{p}_A(q) - F_B^{\rm W} \hat{p}_B(q)\right).
\end{equation}
Here $\hat{n}(q)$ and $\hat{p}_A(q)$, $\hat{p}_B(q)$ denote the direction and points that maximize and minimize the expression in \eqref{eqn:sd_maxmin}, respectively, which depend on the configuration $q$.

% \begin{equation}
%     (p_A_*(q),p_B_*(q)) = \underset{\substack{p_A \in A \\p_B \in B}}{\operatorname{argmin}} \max_{\norm{\hat{n}} = 1} \hat{n}\cdot \left(F_A^{\rm W}(q) p_A - F_B^{\rm W} p_B\right)
% \end{equation}

% \begin{equation}
%     \hat{n}_*(q) = \underset{\norm{\hat{n}} = 1}{\operatorname{argmax}} \min_{\substack{p_A \in A \\p_B \in B}}\hat{n}\cdot \left(F_A^{\rm W}(q) p_A - F_B^{\rm W} p_B\right)
% \end{equation}

It is important to note that in Euclidean space, signed distance, $h$, is differentiable almost everywhere, and satisfies
% $\norm{\nabla \mathrm{sd}(A,B)}_2 = 1$
$\norm{\frac{\partial h}{\partial p_A}}_2 = 1$
\cite{sakai1996riemannian}.
There exists, however, a set of measure zero where $\frac{\partial h}{\partial q}$ is discontinuous, since  functions $\hat{n}$ and $\hat{p}_A$, $\hat{p}_B$ are nonsmooth due to the $\max$ and $\min$ operators in \eqref{eqn:sd_maxmin}.
Since the above framework requires continuously differentiable $h$, we take special care in applying the theory, and we handle nonsmoothness under the following construction.

First, we
% separate the continuous and discontinuous expressions in $\frac{\partial h}{\partial q}$, by
express the gradient of $h$ as follows:
% Fixing $\hat{n}$, $p_A$, and $p_B$, the gradient of this function can be written as 
\begin{equation} \label{eq:cbf_sd_grad_d}
    \frac{\partial h}{\partial q} = \hat{n}(q)^\top J_{A}(q) + \delta(q),
\end{equation}
where $J_{A}(q) = \frac{\partial F_A^{\rm W}}{\partial q} \hat{p}_A(q)$ and $\delta(q)$ is the remainder term associated with the derivatives of $\hat{n}$, $\hat{p}_A$, and $\hat{p}_B$.
Importantly, note that
% the coordinate transformation $F_A^{\rm W}$ is continuously differentiable, hence
$\hat{n}(q)^\top J_{A}(q)$ is continuous, while $\delta(q)$ is discontinuous on a set of measure zero.
The term $\hat{n}(q)^\top J_{A}(q)$ can be interpreted as a continuous approximation of $\frac{\partial h}{\partial q}$,
% by neglecting that $\hat{n}$, $\hat{p}_A$, and $\hat{p}_B$ may vary as the configuration $q$ changes.
while the approximation error $\delta(q)$ acts as disturbance.
The size of the disturbance is characterized by its essential supremum\footnote{The function $\delta$ is \textit{essentially bounded} if ${\Vert \delta(t)\Vert}_2$ is bounded by a finite number for almost all $t \geq t_0$ (i.e., ${\Vert \delta(t)\Vert}_2$ is bounded except on a set of measure zero). The quantity ${\Vert \delta \Vert}_\infty$ is then defined as the least such bound.}:
$$
\| \delta\|_{\infty} := \mathrm{ess} \sup_{\hspace{-.4cm} t \geq t_0} {\| \delta(q(t)) \|}_2. $$
The points where $h$ is not differentiable and $\delta$ is discontinuous occur on a set of measure zero, and therefore do not impact the essential supremum.
% To obtain a bound for $\norm{\delta}_\infty$, we first note that in Euclidean space, signed distance, $h$, is differentiable almost everywhere, and satisfies the equation $\norm{\nabla \mathrm{sd}(A,B)}_2 = 1$ \cite{sakai1996riemannian}.

Now we incorporate the continuous approximation
$\hat{n}(q)^\top J_{A}(q)$
in \eqref{eq:cbf_sd_grad_d} into the control design.
% (instead of the discontinuous gradient $\frac{\partial h}{\partial q}$).
The following result demonstrates that this approximation is sufficient to maintain safety if the disturbance $\delta(q)$ is properly accounted for (in an input-to-state safety (ISSf) context \cite{kolathaya2018input,alan2021safe}).

\begin{prop} \label{prop:issf}
Consider the kinematic model of a robotic manipulator \eqref{eqn:kinematic_model}.
Then, the controller expressed as the QP: 
\begin{align}
\label{eqn:QPkinematic2}
v^*(q,t) = \underset{v \in \mathbb{R}^n}{\operatorname{argmin}} & ~  ~ {\| v - v_{\rm des}(q,t) \|}_2^2 \\
\mathrm{s.t.} & ~  ~ \hat{n}(q)^\top J_{A}(q) v \geq - \alpha h(q) + 2 J_{\max}  \dot{q}_{\max} ,
\nonumber
\end{align}
with $\dot{q}_{\max} = \| \dot{q} \|_{\infty}$ and $J_{\max} = \max_{q \in \R^n} {\| J_{A}(q) \|}_2$,
% \begin{equation}\label{eqn:Jmax}
% J_{\max} = \max_{q \in \R^n} J_{A}(q)
% % \textrm{ and } 
% %   \norm{J_{\max}}_1 = \underset{x \neq 0}{\sup} \frac{\norm{J_{\max}x}_1}{\norm{x}_1}, 
% \end{equation} 
renders the set $\S$ in \eqref{eq:cbf_sd_set} forward invariant for the resulting closed-loop system.
That is, the controller \eqref{eqn:QPkinematic2} keeps system \eqref{eqn:kinematic_model} safe.
\end{prop} 

% The result of this proposition is that, for the kinematic model of the robot manipulator, collision free behavior is guaranteed up to the size of the deviation from the approximation $\hat{n}^\top J_{A}(q)$. 

As such, collision-free behavior is enforced for the kinematic model of the manipulator, if the disturbance, i.e., the approximation error in \eqref{eq:cbf_sd_grad_d}, is accounted for in the controller.
This is achieved by the last term in the constraint of \eqref{eqn:QPkinematic2}. 

\begin{proof}

% The result is that $h$ is a valid input-to-state safe control barrier function (ISSf-CBF) since this inequality is trivially satisfied for $\dot{q} = 0$.  It therefore follows from \cite{ISSf} that: 
% \begin{eqnarray}
% \label{eqn:Sdelta}
% \S_{\delta} := \{ q \in \R^n ~ : ~ h(q) - \frac{\dot{q}_{\max}}{\alpha} \| \delta \|_{\infty} \geq 0 \}
% \end{eqnarray}
% is forward invariant.  
First, we bound the essential supremum $\| \delta\|_{\infty}$ of the disturbance.
Recall that the points where $h$ is not differentiable are on a set of measure zero and do not impact the essential supremum, thus we construct the bound on $\| \delta\|_{\infty}$ by picking generic points where the $h$ is differentiable. 
For an arbitrary point on the robot $p_A \in A$ where $h$ is differentiable:
% , we have:
% Linear velocity at a point $p_A$ is related to joint velocity via 
% \begin{equation}
%     v = J_{A}(q)\dot{q},
% \end{equation}
% thus the maximum possible velocity at any point can be calculated as
% \begin{equation}
%     v_{\max} = \max_{\substack{p_A \in A \\ q \in Q}}J_{A}(q)\dot{q} = J_{\max}\dot{q}.
% \end{equation}
\begin{align}
\begin{split}
    % \frac{\partial h}{\partial q} &= \frac{\partial h}{\partial p_A} \frac{\partial p_A}{\partial q} \\
    \norm{\frac{\partial h}{\partial q}}_2 &= \norm{\frac{\partial h}{\partial p_A} \frac{\partial p_A}{\partial q}}_2\\
     &\leq \norm{\frac{\partial h}{\partial p_A}}_2 \norm{\frac{\partial p_A}{\partial q}}_2 \\
     &\leq 1 \cdot J_{\max}.
\end{split}
\label{eqn:dhdq_bound}
\end{align}
This leads to the bound:
\begin{align}
\begin{split}
    % \delta(q) &= \frac{\partial h }{\partial q} - \hat{n}^\topJ_{A}(q) \\
    \norm{\delta}_\infty &= \norm{\frac{\partial h }{\partial q} - \hat{n}(q)^\top J_{A}(q)}_\infty  \\
    &\leq  \norm{\frac{\partial h }{\partial q} - \hat{n}(q)^\top J_{A}(q)}_2  \\
    &\leq \norm{\frac{\partial h }{\partial q}}_2  + \norm{\hat{n}(q)^\top J_{A}(q)}_2 \\
    &\leq J_{\max}  + \norm{J_{A}(q)}_2 \\
    &\leq 2 J_{\max}.
\end{split}
\end{align}

Then, we differentiate the CBF $h$ in \eqref{eq:cbf_sd} and use \eqref{eq:cbf_sd_grad_d}:
\begin{align}
\begin{split}
\dot{h}(q,\dot{q}) &  = \frac{\partial h}{\partial q} \dot{q}  = \hat{n}(q)^\top J_{A}(q) \dot{q} + \delta(q) \dot{q} \\
& \geq \hat{n}(q)^\top J_{A}(q) \dot{q} - \| \delta \|_{\infty} \dot{q}_{\max}.
\end{split}
\label{eqn:hlowrbound}
\end{align}
% where $\dot{q}_{\max} = \| \dot{q} \|_{\infty}$ and $\delta$ is given as in \eqref{eq:cbf_sd_grad_d}.  
Substituting $\dot{q}$ with the solution $v^*(q,t)$ to \eqref{eqn:QPkinematic2} and incorporating the bound on $\| \delta\|_{\infty}$, the result is: 
\begin{align}
\begin{split}
\dot{h}(q,v^*(q,t)) & \geq 
\hat{n}(q)^\top J_{A}(q) v^*(q,t) - \| \delta \|_{\infty} \dot{q}_{\max} \\
& \geq -\alpha h(q) + 2 J_{\max} \dot{q}_{\max} -  \| \delta \|_{\infty} \dot{q}_{\max} \\
& \geq -\alpha h(q).
\end{split}
\label{eqn:sd_invariance}
\end{align}
Thus, the set $\S$ is forward invariant based on Theorem \ref{thm:CBF}.
\end{proof}

\subsection{Self-collisions}
Self-collisions are defined as collisions between any two links of the robot that are not explicitly allowed to collide. For these types of collisions, we still use the signed distance function, but now $F_B^{\rm W}$ also depends on the configuration $q$:
\begin{equation}
    \mathrm{sd}_{AB}(q) = \max_{\substack{\tilde{n} \in \R^3 \\ \norm{\tilde{n}}_2 = 1}}\min_{\substack{p_A \in A \\p_B \in B}}\tilde{n}\cdot \left(F_A^{\rm W}(q) p_A - F_B^{\rm W}(q) p_B\right).
\end{equation}
Thus, the gradient of $h(q)=\mathrm{sd}_{AB}(q)$ becomes:
\begin{equation} \label{eq:cbf_sd_grad_self}
    \frac{\partial h}{\partial q} = \hat{n}(q)^\top \left( J_{A}(q)-J_{B}(q) \right)  + \delta(q),
\end{equation}
with $J_{A}(q) = \frac{\partial F_A^{\rm W}}{\partial q} \hat{p}_A(q)$ and $J_{B}(q) = \frac{\partial F_B^{\rm W}}{\partial q} \hat{p}_B(q)$.

Proposition~\ref{prop:issf} can again be applied to self-collisions, with slight modifications. The analysis results in the QP:
\begin{align}
\label{eqn:QPkinematic3}
v^*(q,t) &= \underset{v \in \mathbb{R}^n}{\operatorname{argmin}} ~  ~ {\| v - v_{\rm des}(x,t) \|}_2^2 \\
\mathrm{s.t.} & ~  ~ \hat{n}(q)^\top \left( J_{A}(q)-J_{B}(q) \right) v \geq - \alpha h(q) + 4 J_{\max} \dot{q}_{\max} . \nonumber
\end{align}

\subsection{Safety Guarantees for the Full-Order Dynamics}
\label{sec:fullordercollision}

The safety guarantees of Proposition \ref{prop:issf} are valid for the kinematic model \eqref{eqn:kinematic_model}.
% of the manipulator. 
However, like in Theorem \ref{thm:kintofull}, the controllers \eqref{eqn:QPkinematic2} and \eqref{eqn:QPkinematic3}
% that generate safe velocity commands
lead to collision-free motion also on the full-order dynamics---assuming good velocity tracking. 

\begin{thm}
Consider the full-order dynamics of a robot manipulator \eqref{eqn:robot} expressed as the control system \eqref{eqn:dyn}, and the safe set $\S$ in \eqref{eq:cbf_sd_set} associated with the signed distance $\mathrm{sd}_{AB}(q)$ between the robot and the environment in \eqref{eqn:sd_maxmin}.
Let $v^*(q,t)$ be the safe velocity given by the QP \eqref{eqn:QPkinematic2}, with corresponding error in \eqref{eqn:error}.
If Assumption~\ref{ass:tracking} holds with
$\lambda > \alpha$, safety is achieved for the full-order dynamics \eqref{eqn:robot} in that: 
\begin{eqnarray}
(q_0,\dot{e}_0) \in \S_M ~ \Rightarrow ~ q(t) \in S, \quad \forall t \geq t_0,
\end{eqnarray}
where:
\begin{equation}
\S_M  = \left\{ (q,\dot{e}) \in \R^{2n} ~ : ~ \mathrm{sd}_{AB}(q) - \frac{J_{\max} M}{\lambda - \alpha} \| \dot{e} \|_2 \geq 0 \right\}.
\end{equation}
% $$
% h_M (q,\dot{e})  :=   \mathrm{sd}_{AB}(q) - \frac{J_{\max}}{\lambda - \alpha} M \| \dot{e} \| 
% $$
% and $S_M$ is the associated safe set. 
\end{thm}

Note that the same safety guarantees can be stated for self-collision avoidance with the QP \eqref{eqn:QPkinematic3}.

\begin{proof}
The proof follows the same steps as in the Proof of Theorem~\ref{thm:kintofull} with the substitution $C_h = J_{\max}$, which is justified by ${\left\| \frac{\partial h}{\partial q}\right\|}_2 \leq J_{\max}$ based on \eqref{eqn:dhdq_bound}.
Furthermore, note that $\frac{\partial h}{\partial q} v^* \geq -\alpha h(q)$ still holds due to \eqref{eqn:sd_invariance}.
\end{proof}

% The goal of this section is to establish that these same control methodology gives safety guarantees on the full-order dynamics assuming good tracking of the desired velocity. 

% Consider the dynamics associated with a robotic manipulator \cite{murray2017mathematical}: \begin{equation}
%         D(q) \dot{q} + C(q,\dot{q})\dot{q}+G(q) = Bu,
% \end{equation}
% with $q,\dot{q} \in \R^n$, $D(q)$ the inertia matrix, $C(q,\dot{q})$ the coriolis matrix, and $G(q)$ is the gravity vector.  Here we assume full actuation and, therefore, the actuation matrix $B$ is invertible and $u \in \R^n$.  Associated with these dynamics is a control system of the form given in \eqref{eqn:dyn} with $x = (q,\dot{q})$ (hence $k = 2n$). 

% Motivated by the approach in \cite{molnar2021model}, we assume a ``good'' low-level controller on the robot manipulator (as is common on industrial robots).  Concretely, given a desired velocity $\dot{q}^*$, consider the corresponding error in tracking this velocity: 
% $$
% \dot{e} = \dot{q} - \dot{q}^*. 
% $$
% We assume the existence of a low-level controller on the manipulator $u = k(q,\dot{q})$ that achieves exponentially stable tracking.  Under this assumption, we have the following main result of the paper which we state in terms of general terms before specializing to the case of avoiding collisions. 

% : $\dot{e}(t)\| \leq M e^{- \lambda t} \| \dot{e}(0) \|$, for $M, \lambda > 0$.  Additionally, we assume that this exponential tracking can be certified by a Lyapunov function:  

\section{Software Implementation and Simulation}
\label{sec:implementation}
\subsection{CBF Implementation on Precomputed Trajectories}

Assuming the knowledge of a reference trajectory, we now detail the trajectory \emph{safety filter} algorithm. The most straightforward implementation of the QPs \eqref{eqn:QPkinematic2} and \eqref{eqn:QPkinematic3} is to run them in real-time paired with a desired joint velocity controller, which tracks the waypoints of the reference. This can be achieved with a P controller to the next waypoint $i$:
% in the trajectory:
\begin{equation} \label{eq:simplePController}
    v_{\textrm{des}}(q,t) =
    % k(q^i_{\textrm{des}}(t),q) =
    K_P(q^i_{\textrm{des}}-q).
\end{equation}
For the best results, the error on joint positions should be heavily saturated to avoid large differences in desired velocities at short and long distances. The tracked waypoint is iterated forwards either when the robot is sufficiently close $\left( \norm{q^i_{\textrm{des}}-q}_2 < \epsilon \right)$, or when the robot gets stuck. 

Due to the large time delay that many industrial manipulators have, it is often desired to instead send precomputed time-stamped trajectories, rather than attempting to track a trajectory online with feedback. The basic algorithm for generating these safe trajectories, given a cache of previously computed reference trajectories, is detailed in Algorithm \ref{alg:filter}. 

\begin{algorithm}
\caption{Trajectory generation in modified collision environments with safety filters.}\label{alg:filter}
\begin{algorithmic}
\Require $C$, the cache that contains behaviors $C^i_B$, planning scenes $C^i_P$, and trajectories $C^i_X$
\Input
    \Desc{$B$}{Desired behavior}
    \Desc{$P$}{Planning Scene}
    \Desc{$q$}{Robot State}
    % \Desc{C}{Initial cache} \Comment{Contains $C_B$, $C_P$, $C_X$}
\EndInput
\Output
    \Desc{$X$}{Trajectory}
\EndOutput

\ForEach {$C^i$ s.t. $B == C^i_B$} \Comment{Iterate through cache}
    % \State $\delta^i_{X_0} = \norm{C^i_{X_0}-X_0}$
    % \State $\delta^i_{P} \  = \norm{C^i_{P}-P}$
    % \State $T^i = f(\delta^i_{X_0},\delta^i_{P})$ 
    \State $T^i = f(C^i_P,C^i_{X_0}, P, q)$  \Comment{Compute suitability metric}
\If{$T^i < T_1$} \Comment{Reference is extremely similar}
    \State $X \gets$ CBF($C^i_X,P,q$)
    \State \Return
\EndIf
\EndFor
\State [$T_{\min}$, idx]  $ \gets \min(T^i)$ \Comment{Find best reference}
\If{$T_{\min} < T_2$} \Comment{Close match}
    \State $X \gets$ CBF($C^{\textrm{idx}}_X,P,q$) \Comment{Safety filter}
    \State \Return
\ElsIf{$T_{\min} < T_3$} \Comment{Suitable match}
    \State $X \gets$ CBF($C^{\textrm{idx}}_X,P,q$)
    \State $C \gets X$ 
    \State \Return
\Else \Comment{Best reference is very dissimilar}
    \State $X \gets$ Re-plan from scratch
    \State $C \gets X$ \Comment{$X$ gets added to cache}
\EndIf
\end{algorithmic}
\end{algorithm}

% \AAcomment{Note that even if this does not fully align with the Theorem I propose adding, you can make connections with the theory emperically.}

% \begin{algorithm}
% \caption{CBF Filter}\label{alg:cbf}
% \begin{algorithmic}
% \Input
%     \Desc{$X_{\rm{des}}$}{Desired trajectory}
%     \Desc{P}{Planning scene}
%     \Desc{$X_0$}{Initial robot state}
% \EndInput
% \Output
%     \Desc{$X_f$}{Filtered trajectory}
% \EndOutput
% \State $X \gets X_0$
% \State $X_f \gets X$
% \State $i \gets 1$
% \State $j \gets 1$
% \While {$i < \textrm{size}(X_{\rm{des}})$}
%     \State $u_{\rm{des}} \gets k(X^i_{\rm{des}},X)$
%     \State $u_{\rm{act}} \gets CBF(u_{\rm{des}},X,P)$
%     \State $X \gets X + \int_0^{\Delta_t}\left(f(X) + g(X) u_{\rm{act}}\right)dt$
% \If{$\norm{X-X^i_{\rm{des}}} \leq \epsilon$ or $j \geq J_{\max}$} 
%     \State $i \gets i + 1$
%     \State $j \gets 1$
%     \State $X_f \gets [X_f,X]$
% \Else 
%     \State $j \gets j + 1$
% \EndIf
% \EndWhile
% \If{$\norm{X-X^i_{\rm{des}}} \leq \epsilon$} 
%     \State \Return X \Comment{Goal reached}
% \Else
%     \State \Return error \Comment{Goal not reached}
% \EndIf
% \end{algorithmic}
% \end{algorithm}

There are three fields of interest in the cached trajectories: the desired behavior $B$, the manipulator's trajectory $T$, and the collision environment used by the original planner, referred to as the planning scene $P$. While only the joint trajectory is required to generate the modified, safe trajectory, the inclusion of the original planning scene allows for more information when choosing the closest trajectory to track.

The algorithm first assesses the suitability of previously computed trajectories in the cache. There are two major considerations: the difference in initial conditions and the similarity of the planning scene. The suitability of the $i^{\textrm{th}}$ member of the cache $C^i$ is evaluated by the function:
\begin{equation}
    T^i = f(C^i_P,C^i_{X_0}, P, q) = \delta^i_{q} + \delta^i_{P},
\end{equation}
where
\begin{align}
    \delta^i_{q} &= \norm{C^i_{X_0}-q}_2\\
    \delta^i_{P} &= \norm{C^i_{P}-P} = \sum_{o \in O}\norm{C^i_{P_o} - P_o}
\end{align}
assess the differences in the initial conditions of the robot and the collision objects $o \in O$ making up the planning scene.
% where $C^i_{X_0}$ is the initial robot state of the cache member, and $C^i_{P}$ is its corresponding planning environment, which consists of collision objects $o \in O$ with positions given by $C^i_{P_o} \in \R^3$. In the cooking environment we are considering, the differences in the orientations of obstacles in the environment are negligible, but in some settings, it may be prudent to include orientation changes inside of $\delta^i_{P}$.

There are three threshold values ($T_1$, $T_2$ and $T_3$) for this suitability metric. If $T^i < T_1$, then the search stops, as the trajectory in the cache is so close that it is not worth searching, and the CBF filter is applied. After searching through all cache members, if $T^i < T_2$, then the filter is applied, but the trajectory is not added to the cache to prevent it from growing unnecessarily large. If $T_2 < T^i < T_3$, then the filter is applied and the resulting trajectory is added to the cache. Finally, if $T^i > T_3$, then the original motion planning algorithm is used, and the result is added to the cache. 

To obtain the joint trajectory $X$ via the CBF, we simply utilize a trajectory tracking controller like \eqref{eq:simplePController} along with the CBF-QP, and integrate its solution throughout the behavior.

\subsection{Software Implementation and Simulation}

Figure~\ref{fig:cooking_env} shows the simulated cooking environment.
The robot and obstacle representations are a series of meshes described by URDF and SRDF files.
The position and orientation of objects are updated before each planning attempt, and collision objects in the environment are assumed to be stationary unless directly interacted with by the manipulator, such as the baskets being grabbed and moved. 

\begin{figure}[t]
\includegraphics[width=\columnwidth]{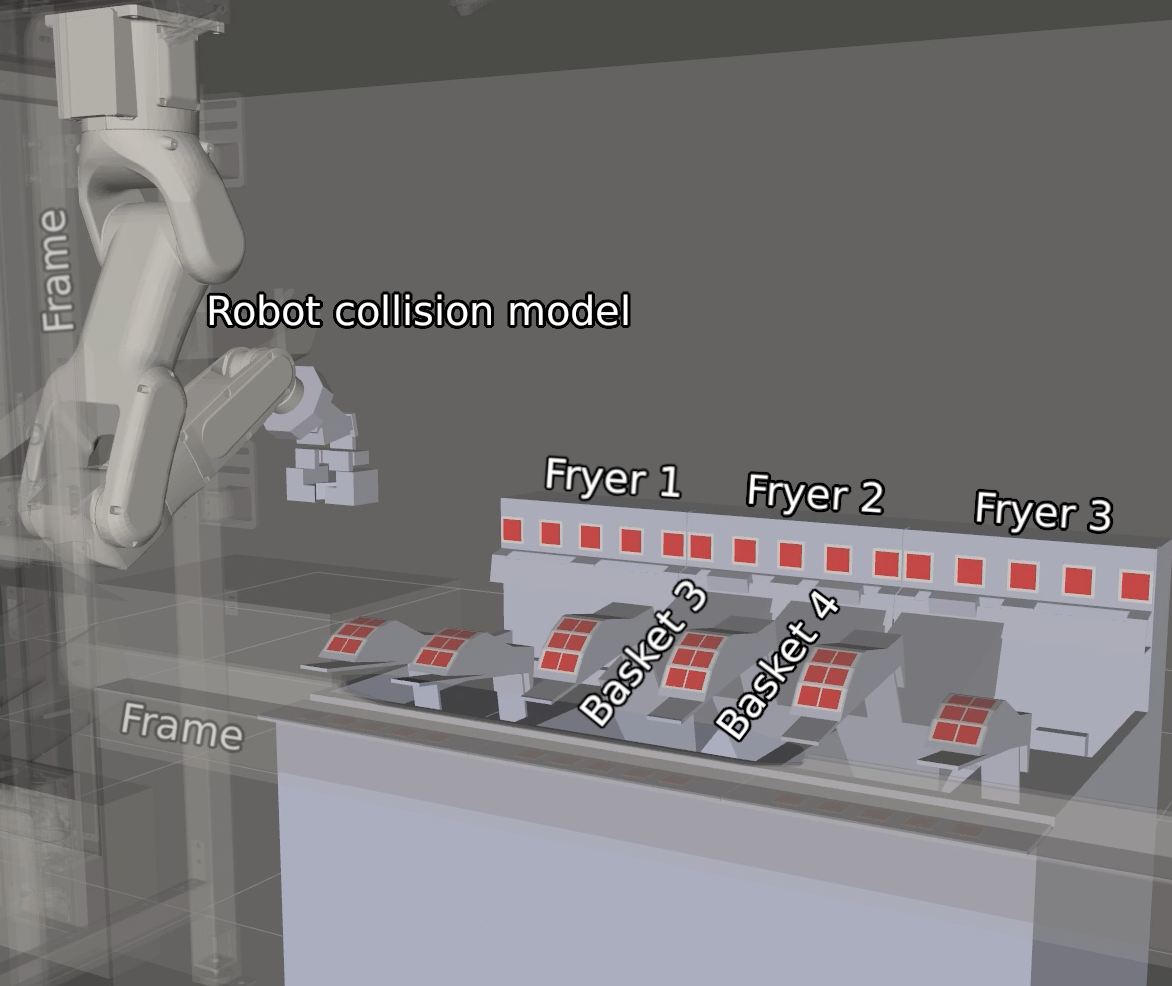}
\caption{The simulation environment, which shows the collision objects and their representations as mesh files. The same mesh representations are used on the hardware system.} \label{fig:cooking_env}
\end{figure}

To implement the CBF filter, we require three values to be computed: the signed distance to the obstacles and other links $\textrm{sd}(q)$, the normal vectors corresponding to these points $\hat{n}(q)$, and the manipulator Jacobian at these points $J(q)$. 
The MoveIt framework \cite{coleman2014reducing}, an open-source robotics software package for motion planning, is able to compute all three of these values. Specifically, the \texttt{distanceRobot()} and \texttt{distanceSelf()} functions of the \texttt{CollisionEnv} class provide the signed distances and normal vectors needed for environmental and self-collisions. Moreover, the \texttt{getJacobian()} function in the \texttt{RobotState} class returns the manipulator Jacobian. Thus, no other external libraries are required to implement this algorithm.
Once these three values are computed, the OSQP quadratic program solver \cite{stellato2020osqp} is used to calculate the velocity commands subject to the CBF condition, and integration is done manually.

Before hardware implementation, the algorithm was tested in simulation. The resulting behaviors are described in the next section, and the simulation results are shown along with the hardware trajectories in Figure \ref{fig:hardware_results}.

\section{Hardware Results}
\label{sec:results}
\subsection{Experimental testing environment}

\begin{figure*}[t]
    \centering
    % \begin{subfigure}{0.49\textwidth}
    % \includegraphics[width=1\columnwidth]{example-image-a}
    % \caption{\texttt{\texttt{fryer\_to\_hanger}} with adjacent basket hanging (simulation)z.}
    % \end{subfigure}
    % \rulesep
    \begin{subfigure}{1\textwidth}
    % \rulesep
    \includegraphics[width=1\columnwidth]{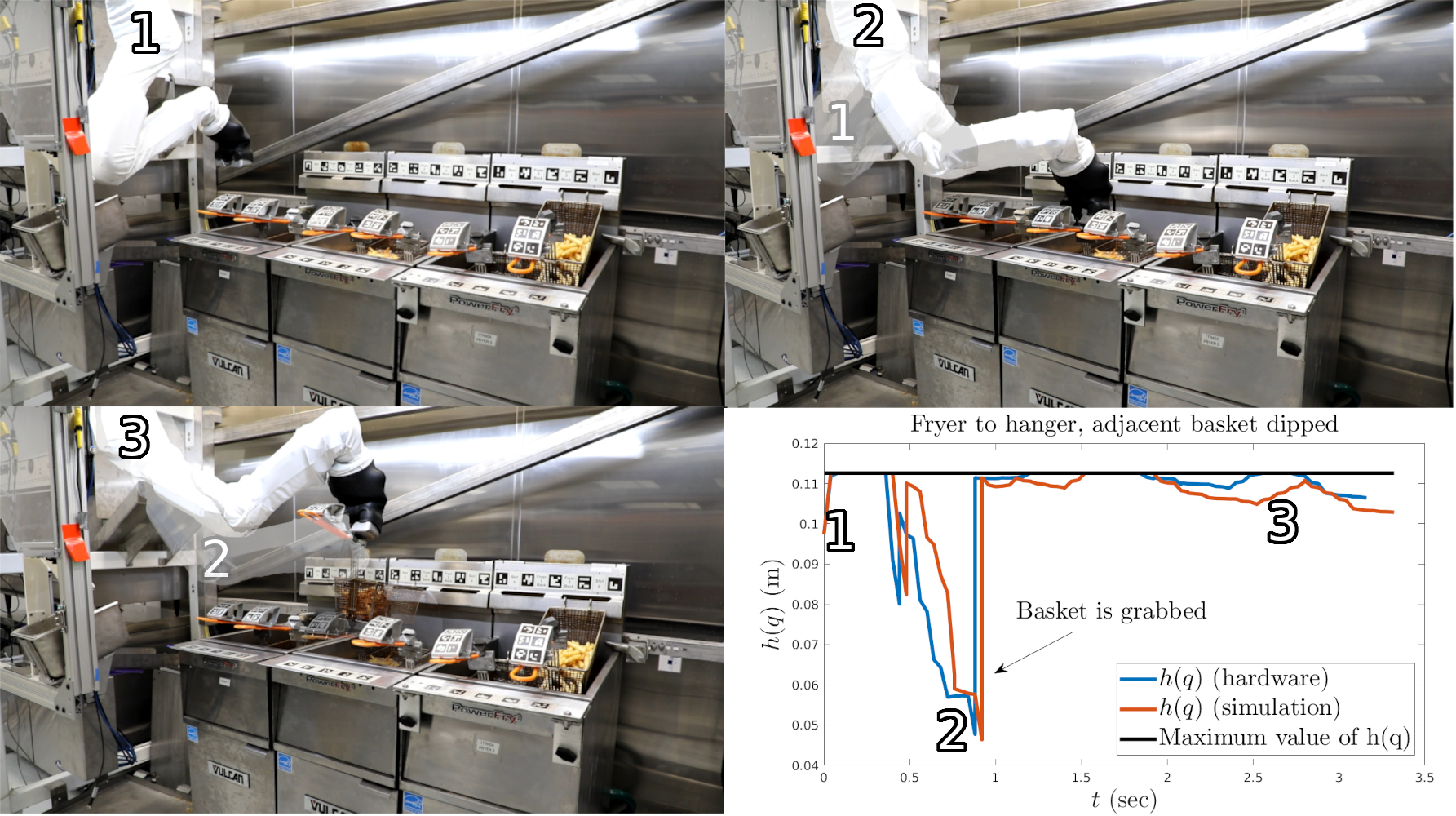}
    \caption{\texttt{\texttt{fryer\_to\_hanger}} with adjacent basket in fryer.}
    \end{subfigure}
    \vspace{.1cm}
    \begin{subfigure}{1\textwidth}
    \includegraphics[width=1\columnwidth]{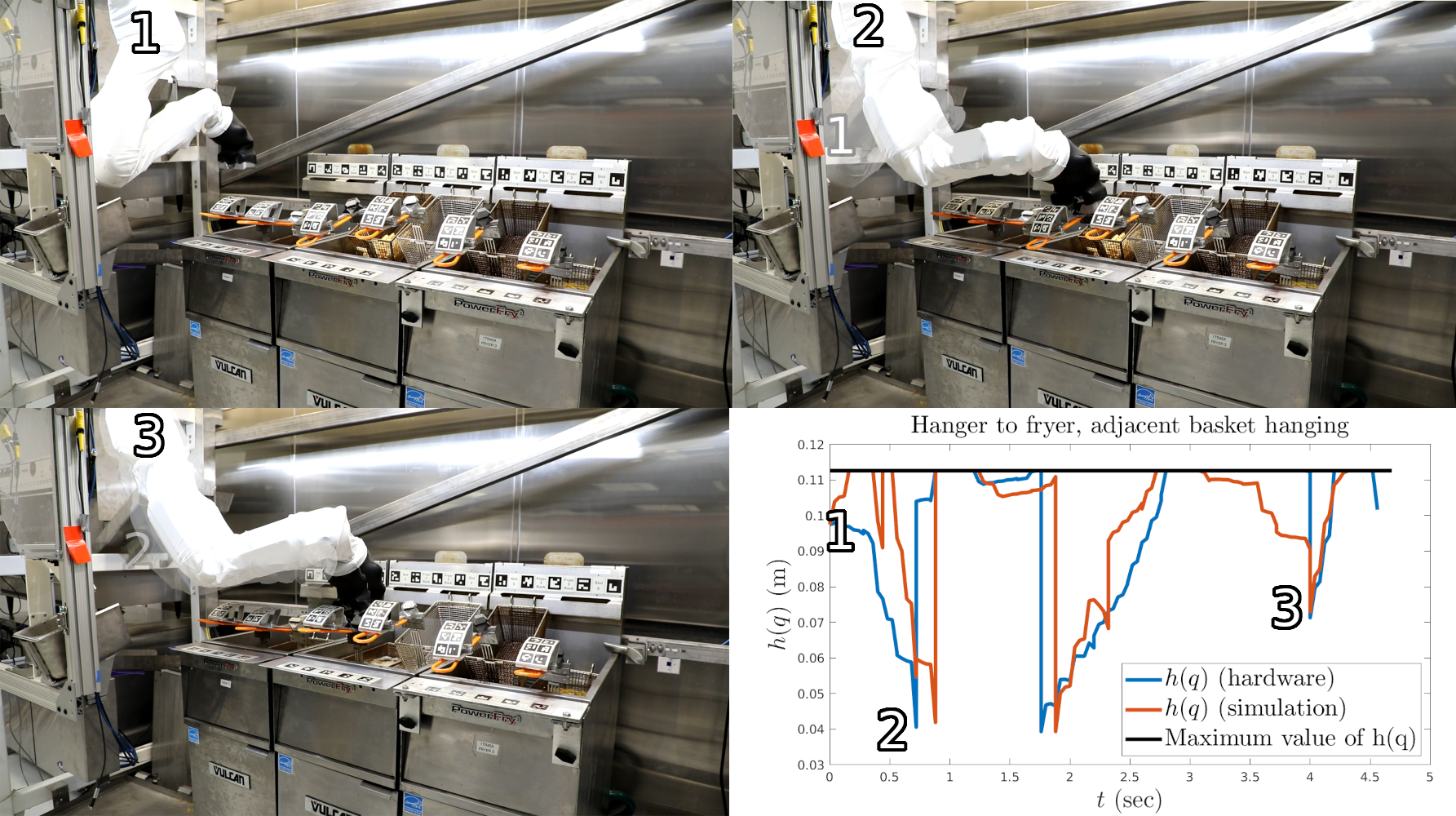}
    \caption{\texttt{\texttt{hanger\_to\_fryer}} with adjacent basket hanging.}
    \end{subfigure}
    \caption{Two examples behaviors implemented on the Flippy2 robot. See \url{https://youtu.be/nmkbya8XBmw} for video. The large spikes in signed distance $h(q)$ come from enabling and disabling collision objects when required for interaction, like the basket when gripping and the fryer when hanging. At the maximum value of $h(q)$, the robot is only 11 cm away from the frame around it during these behaviors.}\label{fig:hardware_results}
\end{figure*}

We apply the approach described in this paper to one of the Miso Robotics robotic cooking environments. Specifically, we utilize a FANUC LR Mate 200iD/7LC robotic manipulator wrapped in a sleeve, and we send joint trajectories from an Intel i9-9900KF running ROS.

The cooking environment used in the testing is fully modeled using high-quality meshes used for collision checking. There are 36 collision objects in total, each represented by tens to hundreds of mesh triangles. The primary collision objects of concern are the six baskets, three industrial fryers, the hood vent over the fryers, and the glass pane separating the manipulator from the human workers. Of these objects, the baskets and fryers are the most commonly displaced. 

As shown in the figures, the configuration space of the manipulator is very densely crowded with obstacles. To complete a behavior, it is common to have less than a few centimeters of clearance between the robot and the surrounding environment. For this reason, planning methods must be minimally conservative, and there is no room for any collision buffer. 

For the purpose of the experiments, a minimal cache was utilized to highlight the role of CBFs in re-planning around obstacles. In a commercial setting, with a more populated cache, the CBF would have many more prior trajectories to choose from, meaning that the path modifications would be much smaller in magnitude. In practice, we find that the cache size saturates at around 200 stored behaviors.

\subsection{Hardware results}

We test our framework's ability to safely re-plan on the two most volatile behaviors: \texttt{fryer\_to\_hanger} and \texttt{hanger\_to\_fryer}, described below.

\noindent \textbf{Fryer to hanger.} The \texttt{fryer\_to\_hanger} behavior moves a basket from the dipped state to the hanging state. The manipulator picks up a basket that has finished cooking and hangs it, allowing the oil to drip off the basket before serving food to customers.

\noindent \textbf{Hanger to Fryer.} The \texttt{hanger\_to\_fryer} behavior is the reverse of \texttt{fryer\_to\_hanger}, transitioning a basket from the hanging state to the frying state.

Each behavior is tested in two primary configurations: one where the adjacent basket is submerged, and one where it's hanging. For the purpose of this paper, each of the four testing configurations were run 25 times, each with different cached trajectories and planning environments, for 100 total executions. The testing methodology was simple: for each setup, we first run the CBF on the best matching reference trajectory in the limited cache, and then we re-plan using TrajOpt for comparison purposes. 

\textit{The CBF was able to produce a successful, collision-free trajectory in all 100 cases}, even with the artificially limited cache size. The average computation time per CBF call was 2 ms, and the average computation time for the entire behavior was 223 ms. This is a significant improvement compared to TrajOpt's average computation time of 5923. Note that the CBF's trajectory is updated every 10 ms compared to TrajOpt's 64 ms, meaning no additional local planner needs to be utilized.
% The timing results are summarized in Table \ref{tab:hardware}, 
Two example trajectories from the CBF are visualized in Figure \ref{fig:hardware_results}, and the value of $h(q)$ throughout the motion is included. 
% \begin{table} 
% \caption{Timing and trajectory density of CBF and TrajOpt}\label{tab:hardware}
% % \begin{tabular}{ |p{3cm}||p{3cm}|p{3cm}|p{3cm}|  }
% \begin{tabular}{ |c||c|c||c|c|c|  }
%  \hline
%  Method & 
%  \multicolumn{2}{|c|}{CBF} &
%  \multicolumn{2}{|c|}{Trajopt}\\
%  \hline
%  Behavior & T (ms) & $\Delta_t$ ($1/$s) & T (ms) & $\Delta_t$ ($1/$s) \\
%  \hline
%   Fryer$\rightarrow$hanger (fryer)   & 182    &0.01    &5040     &0.064\\
%   Fryer$\rightarrow$hanger (hanger)  & 178    &0.01    &4783       &0.064\\ 
%   Hanger$\rightarrow$Fryer (fryer)   & 266    &0.01    &7039     &0.064\\
%   Hanger$\rightarrow$Fryer (hanger)  & 262    &0.01    &6832       &0.064\\ 
%  \hline
% \end{tabular}
% \end{table}

\section{Conclusion}
\label{sec:conclusion}
In this work, we showcased control barrier functions for utilization in complex, real-world collision environments in the case of robotic cooking applications. First, we demonstrated how CBFs applied to the kinematics of robotic manipulators guarantee safety for the full-order dynamics. Then, we described the construction of these CBFs for very complex collision obstacle representations. We proposed an algorithm for filtering reference trajectories via CBFs to achieve safety, and we demonstrated the capabilities of this method in software and on hardware in the real-world application of frying foods.

% \IEEEtriggeratref{12}
% \renewcommand{\baselinestretch}{0.98}
\bibliographystyle{IEEEtran}
\bibliography{refs.bib}

\end{document}